\documentclass[letterpaper]{article} 
\usepackage{aaai24}  
\usepackage{times}  
\usepackage{helvet}  
\usepackage{courier}  
\usepackage[hyphens]{url}  
\usepackage{graphicx} 
\usepackage{natbib}  
\usepackage{caption} 
\usepackage{algorithm}
\usepackage{algorithmic}
\usepackage{graphicx}
\usepackage{amsmath}
\usepackage{amssymb}
\usepackage{booktabs}
\usepackage{cite}
\usepackage{bm}
\usepackage{multicol}
\usepackage{multirow}
\usepackage[table]{xcolor}
\usepackage{amsthm} 
\newtheorem{theorem}{Theorem}
\usepackage{newfloat}
\usepackage{listings}
\DeclareCaptionStyle{ruled}{labelfont=normalfont,labelsep=colon,strut=off} 
\floatstyle{ruled}
\newfloat{listing}{tb}{lst}{}
\floatname{listing}{Listing}
\title{DS-AL: A Dual-Stream Analytic Learning for Exemplar-Free Class-Incremental Learning}
\author {
Huiping Zhuang\textsuperscript{\rm 1},
Run He\textsuperscript{\rm 1},
Kai Tong\textsuperscript{\rm 1},
Ziqian Zeng\textsuperscript{\rm 1},
Cen Chen\textsuperscript{\rm 2,3}\thanks{Corresponding author},
Zhiping Lin\textsuperscript{\rm 4}
}
\affiliations {
\textsuperscript{\rm 1}Shien-Ming Wu School of Intelligent Engineering, South China University of Technology, China\\
\textsuperscript{\rm 2} School of Future Technology, South China University of Technology, China\\
\textsuperscript{\rm 3} Pazhou Laboratory, China\\
\textsuperscript{\rm 4} School of Electrical and Electronic Engineering, Nanyang Technological University,  Singapore\\
\{hpzhuang, zqzeng, chencen\}@scut.edu.cn, \{wirh, wikaitong\}@mail.scut.edu.cn, ezplin@ntu.edu.sg
}

\begin{document}

\maketitle

\begin{abstract}
	Class-incremental learning (CIL) under an exemplar-free constraint has presented a significant challenge. Existing methods adhering to this constraint are prone to catastrophic forgetting, far more so than replay-based techniques that retain access to past samples. In this paper,  to solve the  exemplar-free CIL problem, we propose a Dual-Stream Analytic Learning (DS-AL) approach. The DS-AL contains a main stream offering an analytical (i.e., closed-form) linear solution, and a compensation stream improving the inherent under-fitting limitation due to adopting linear mapping. The main stream redefines the CIL problem into a Concatenated Recursive Least Squares (C-RLS) task, allowing an equivalence between the CIL and its joint-learning counterpart. The compensation stream is governed by a Dual-Activation Compensation (DAC) module. This module re-activates the embedding with a different activation function from the main stream one, and seeks fitting compensation by projecting the embedding to the null space of the main stream's linear mapping. Empirical results demonstrate that the DS-AL, despite being an exemplar-free technique, delivers performance comparable with or better than that of replay-based methods across various datasets, including CIFAR-100, ImageNet-100 and ImageNet-Full. Additionally, the C-RLS' equivalent property allows the DS-AL to execute CIL in a phase-invariant manner. This is evidenced by a never-before-seen 500-phase CIL ImageNet task, which performs on a level identical to a 5-phase one. Our codes are available at \url{https://github.com/ZHUANGHP/Analytic-continual-learning}.
	
\end{abstract}

\section{Introduction}
\label{section_introduction}
Class-incremental learning (CIL) \cite{,LwF2018TPAMI,iCaRL2017_CVPR} updates a network's parameters in a phase-by-phase manner, with data arriving separately in each training phase. CIL has gained popularity for its ability to adapt trained models to unseen data categories, reducing energy consumption. Its development is emphasized by the fact that data and target categories are often available at specific locations or time slots. Furthermore, CIL is intuitively inspired by the human learning process, where individuals build upon their existing knowledge by assimilating new information continuously.

CIL offers several benefits but can also lead to \textit{catastrophic forgetting} \cite{cil_review2021NNs}, where models quickly lose previously learned knowledge when acquiring new tasks. This issue is exacerbated under an exemplar-free constraint, where previous experience cannot be stored or revisited. Researchers have proposed various solutions, including replay-based CIL and exemplar-free CIL (EFCIL). Replay-based methods perform competitively but violate the exemplar-free constraint. EFCIL techniques honor data privacy but often yield inferior results.

The EFCIL family consists of a few sub-branches, with regularization-based methods being the most prevalent branch. These methods reduce catastrophic forgetting by imposing constraints to prevent the change of important weights or activations. However, they are usually inadequate as their performances cannot match those of replay-based counterparts \cite{RMM2021NeuriPS}. Recently, a new branch of CIL, the analytic learning (AL) based CIL \cite{ACIL2022NeurIPS}, has shown promising results even while adhering to the exemplar-free constraint. 

The AL-based methods identify the iterative mechanism as the primary cause of catastrophic forgetting and replace it with linear recursive tools. For the first time, they provide results on par with those of replay-based techniques. Additionally, the recursive operation allows the AL-based CIL to give strong performance especially under large-phase scenarios (e.g., 50-phase). However, existing AL-based methods may suffer from an under-fitting dilemma because they rely solely on one linear projection. This has motivated the exploration of potential approaches to compensate for the limited fitting power of vanilla linearity inherited from the AL-based CIL.

In this paper, we introduce a Dual-Stream Analytic Learning (DS-AL). The DS-AL contains a main stream offering an analytical (i.e., closed-form) linear solution, and a compensation stream that improves the inherent under-fitting limitation due to adopting linear mapping. The DS-AL develops the AL-based CIL family by compensating the lack of fitting power, without losing the basic benefits inherited from this category. The key contributions are summarized as follows. 

$\bullet$ We present the  DS-AL, an exemplar-free technique that offers an analytical solution to the CIL problem.

$\bullet$ The DS-AL's main stream redefines the CIL problem into a Concatenated Recursive Least Squares (C-RLS) task, allowing an equivalence between the CIL and its joint-learning. Therefore, models trained in a CIL manner yield identical results to those employing data from both current and historical phases concurrently.

$\bullet$ The DS-AL compensation stream introduces a Dual-Activation Compensation (DAC)  module. This enables our method to overcome the under-fitting limitation inherited from the AL-based CIL.

$\bullet$ Our experiments on benchmark datasets show that the DS-AL, despite being an EFCIL technique, delivers performance better than that of existing AL-based techniques and even surpasses most replay-based methods. Additionally, we conduct a  never-before-seen 500-phase task on ImageNet. It achieves a near-identical result to its 5-phase counterpart, further highlighting DS-AL's equivalent property.

\section{Related Works}
In this section, we review CIL techniques related to our proposed method, which fall into two categories: replay-based and EFCIL techniques. 

\subsection{Replay-based CIL}
Replay-based CIL reinforces models' memory of past knowledge by replaying historical experience. The replay mechanism was first introduced by iCaRL \cite{iCaRL2017_CVPR} and has gained popularity with various attempts follow due to its competitive performance. For instance, end-to-end incremental learning \cite{EEIL2018_ECCV} introduces balanced training via replaying. Bias correction \cite{BiC2019_CVPR} includes an extra trainable layer. LUCIR \cite{LUCIR2019_CVPR} replaces the softmax layer with a cosine one. PODNet \cite{podnet2020ECCV} implements spatial-based distillation loss. FOSTER \cite{FOSTER2022ECCV} adopts a two-stage learning by expanding and reducing the network.

There are also plug-in techniques that can be attached to existing replay-based methods, achieving state-of-the-art (SOTA) performance. For instance, AANets \cite{AANet_2021_CVPR} incorporates stable and plastic blocks to balance stability and plasticity, improving performance when plugged into techniques such as PODNet. Similarly, the Mnemonics technique \cite{Mnemonics_2020_CVPR} explores exemplar storing mechanism. Reinforced memory management (RMM) \cite{RMM2021NeuriPS} leverages reinforcement learning and constructs dynamic memory management for exemplars, achieving outstanding results when attached to PODNet and AANets. Online hyperparameter optimization \cite{Online2023AAAI} adaptively optimizes the stability-plasticity balance without prior knowledge.

In general, replay-based CIL achieves satisfactory results with the need to store previously visited samples.

\subsection{Exemplar-free CIL}
EFCIL methods branch roughly into regularization-based CIL, prototype-based CIL, and AL-based CIL.

\subsubsection{Regularization-based CIL}
Regularization-based methods introduce additional constraints to mitigate forgetting by constructing new loss functions. These constraints can be applied to network activations and penalize changes of important ones. Examples include less-forgetting learning \cite{LFL2016}, which penalizes activation difference, and learning without forgetting (LwF) \cite{LwF2018TPAMI}, which prevents activation changes between old and new networks. Alternatively, regularization can be imposed on network parameters, such as elastic weight consolidation (EWC) \cite{EWC2017nas}, which captures prior importance by a diagonally approximated Fisher information matrix. Building upon EWC, \cite{EWC2_2018ICPR} seeks a more appropriate replacement of the Fisher information matrix.

\subsubsection{Prototype-based CIL}
Prototype-based CIL mitigates forgetting by storing prototypes for each class to avoid overlapping representations of new and old classes. For instance, PASS \cite{PASS2021CVPR} augments the prototypes of features to distinguish previous classes. SSRE \cite{SSRE2022CVPR} introduces a prototype selection mechanism that selectively incorporates new samples into the distillation to enhance the distinctiveness between old and new classes. Fetril \cite{FeTrIL2023WACV} reduces forgetting by generating pseudo-features of old classes from new representations.

\subsubsection{Analytic Learning based CIL}
AL-based methods are newly developed exemplar-free tools specifically designed to protect data privacy during incremental learning. They are inspired by AL \cite{pil2001,brmp2021}, where the training of neural networks yields a closed-form solution using least squares (LS). Analytic CIL (ACIL) \cite{ACIL2022NeurIPS} was first developed, reformulating the CIL procedure into a recursive analytic learning process. The need for storing exemplars is evicted by preserving a correlation matrix. ACIL is subsequently followed by Gaussian Kernel Embedded Analytic Learning (GKEAL) \cite{GKEAL2023CVPR2023}. GKEAL specializes in few-shot CIL settings by adopting a Gaussian kernel process that excels in data-scarce scenarios.

AL-based CIL is an emerging CIL branch, exhibiting strong performance even at its first proposal as ACIL. However, relying solely on one linear projection would likely lead to under-fitting. In this paper, we overcome this limitation by introducing DS-AL, where the DAC module compensates for the lack of fitting ability.

\begin{figure*}[t!]
	\centering
	\includegraphics[width=1\linewidth]{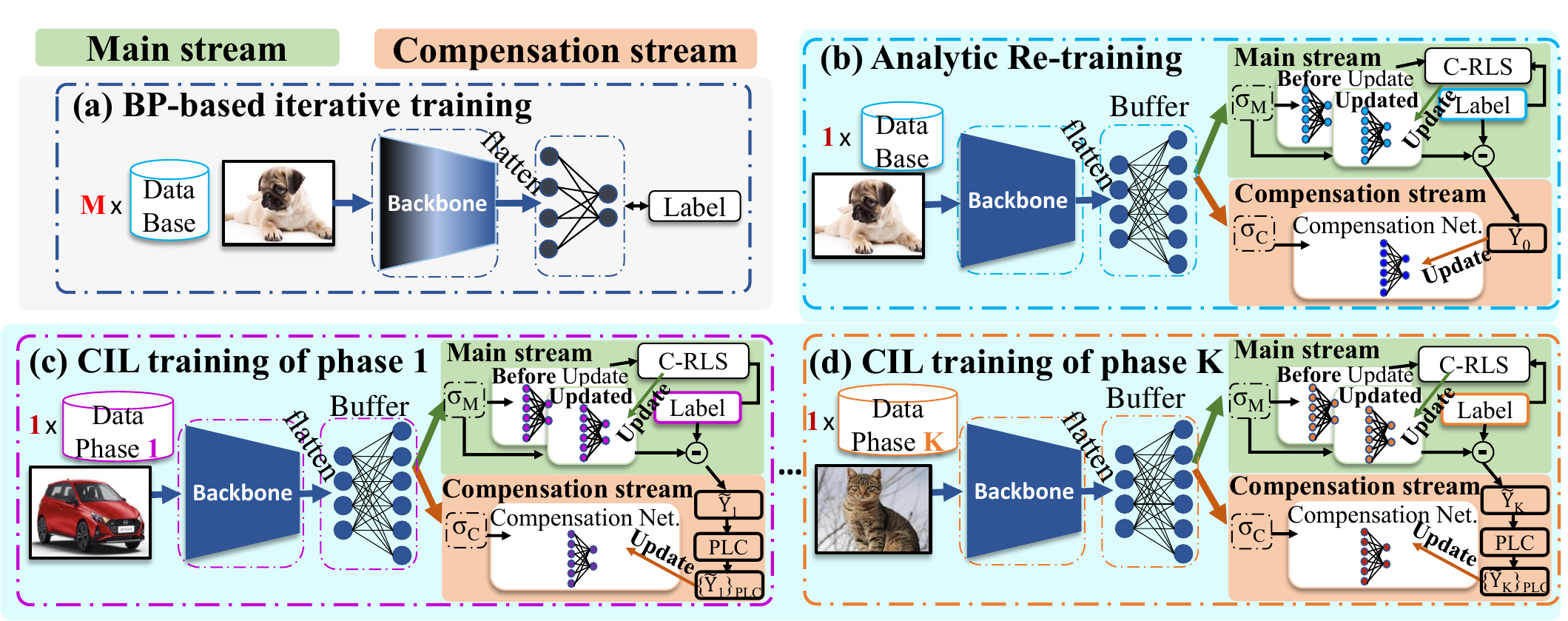}
	\caption{The DS-AL consists of (a) BP-based training on the base dataset, followed by (b)-(d) AL-based training steps. Each step includes a main stream (green block) with a C-RLS module and a compensation stream (orange block) using the mapping residue from main stream as the label. (b) DS-AL initializes CIL by replacing and re-training the classification head with an AL-based one that includes a buffer layer and a linear classifier. (c)-(d) The CIL is then recursively conducted, incorporating a PLC module (defined in \eqref{eq_plc}) to ensure the incremental constraint in the compensation stream.}
	\label{fig:acilflow}
\end{figure*}

\section{The Proposed Method}
In this section, we present our DS-AL training algorithm (see Figure \ref{fig:acilflow}). This starts with a backpropagation (BP) based backbone training on the base dataset (see Figure \ref{fig:acilflow}(a)). The AL-based training steps follow, including an AL-based re-training (i.e., Figure \ref{fig:acilflow}(b)) and follow-up CIL procedures (i.e., Figure \ref{fig:acilflow}(c)-(d)). We focus on the AL-based training steps, each of which is delivered in a dual-stream manner, with the main stream contributed by the C-RLS and the compensation stream governed by the DAC module. 


\subsection{BP-based Training}
The network is first trained with an iterative BP on the base dataset (see Figure \ref{fig:acilflow}(a)) for multiple epochs with an appropriate learning scheduling. For simplicity, we only discuss convolutional neural networks (CNN) examples. Let $\mathbf{W}_{\text{CNN}}$ and $\mathbf{W}_{\text{FCN}}$ be the parameters representing the CNN backbone and the fully-connected classifier. After the BP training, given an input $\mathbf{X}$, the output of the network is 
\begin{align}
	\mathbf{Y} = f_{\text{softmax}}(f_{\text{flat}}(f_{\text{CNN}}(\mathbf{X}, \mathbf{W}_{\text{CNN}}))\mathbf{W}_{\text{FCN}})
\end{align} 
where $f_{\text{CNN}}(\mathbf{X}, \bm{W}_{\text{CNN}})$ indicates the CNN output; $f_{\text{softmax}}$ and $f_{\text{flat}}$ are softmax function and  flattening operator (i.e., reshaping a tensor into a 1-D vector).

After the BP training, the backbone's weights are obtained. We then \textit{freeze the backbone and replace the classifier with a 2-layer AL network for re-training and the forthcoming CIL steps}. This is delivered in a dual-stream manner as follows.
\subsection{The Main Stream of DS-AL }
Prior to further processing, a few definitions are made. Let the network be incrementally trained for $K$ phases where training data of each phase comes with different classes. Let $\mathcal{D}_{k}^{\text{train}}\sim \{\bm{X}_{k}^{\text{train}}, \bm{Y}_{k}^{\text{train}}\}$ and $\mathcal{D}_{k}^{\text{test}}\sim \{\bm{X}_{k}^{\text{test}}, \bm{Y}_{k}^{\text{test}}\}$ be the training and testing datasets at phase $k$ ($k=0,1,\dots,K$). $\bm{X}_{k}\in\mathbb{R}^{N_{k}\times c\times w \times h}$ (e.g., $N_{k}$ images with a shape of $c\times w \times h$) and $\bm{Y}_{k}^{\text{train}}\in\mathbb{R}^{N_{k}\times d_{y_{k}}}$ (with phase $k$ including $d_{y_{k}}$ classes) are stacked input and label (one-hot) tensors.

\subsubsection{AL-based Re-training.} We first conduct the AL-based retraining (see Figure \ref{fig:acilflow}(b)) on the base training set (represented by $\mathcal{D}_{0}^{\text{train}}\sim \{\bm{X}_{0}^{\text{train}}, \bm{Y}_{0}^{\text{train}}\}$). The AL-based classifier includes a \textit{Buffer} layer and a linear mapping. The buffer layer is a common trick to improve AL-based methods' performance \cite{brmp2021}. To show this, we feed the input $\bm{X}_{0}^{\text{train}}$ through the trained CNN backbone, followed by a flattening operation, to extract features, i.e.,
\begin{align}\label{eq_flatten_fea}
	\bm{X}_{0}^{\text{cnn}} = f_{\text{flat}}(f_{\text{CNN}}(\bm{X}^{\text{train}}_{0}, \bm{W}_{\text{CNN}}))
\end{align}
where $\bm{X}_{0}^{\text{cnn}}\in\mathbb{R}^{N_{0}\times d_{\text{cnn}}}$. Then the buffer layer (denoted by $B$) is inserted to map the feature to a different space. That is, the feature $\bm{X}_{0}^{\text{cnn}}$ is mapped to the \textit{main stream activation} $\bm{X}_{\text{M},0}$ as follows
\begin{align}\label{eq_fea_base}
	\resizebox{0.88\linewidth}{!}{$\bm{X}_{\text{M},0} = \sigma_{\text{M}}(B(f_{\text{CNN}}(\bm{X}_{0}^{\text{train}}, \bm{W}_{\text{CNN}}))=\sigma_{\text{M}}(B(\bm{X}_{0}^{\text{cnn}}))$}
\end{align}
where $\sigma_{\text{M}}$ is an activation function. Here, we adopt \textit{ReLU} following many AL-based techniques \cite{ACIL2022NeurIPS}. 

There are a few choices to construct the buffer layer. The most common one is the random projection \cite{ACIL2022NeurIPS}. This is simply a linear projection to map the feature to a higher-dimension space, i.e., $B(\bm{X}_{0}^{\text{cnn}}) = \bm{X}_{0}^{\text{cnn}}\bm{X}_{\text{B}}$. Another known buffer structure is the Gaussian kernel mapping \cite{GKEAL2023CVPR2023}, which projects the feature to a space spanned by computation of the Gaussian kernel function. The buffer layer selection is not the focus of this paper, so we simply follow the ACIL to adopt the random projection represented by $\bm{X}_{\text{B}}\in\mathbb{R}^{d_{\text{cnn}}\times d_{B}}$. 

The second layer of the AL-based network is a linear layer (denoted by $\bm{W}_{\text{M}}^{(0)}$). This is constructed by linearly mapping the $\bm{X}_{\text{M},0}$ to the label matrix $\bm{Y}_{0}^{\text{train}}$ via solving the following regularized linear problem
\begin{align}\label{eq_0_L2}
	\underset{\bm{W}_{\text{M}}^{(0)}}{\text{argmin}} \quad \left\lVert\bm{Y}_{0}^{\text{train}} - \bm{X}_{\text{M},0}\bm{W}_{\text{M}}^{(0)}\right\rVert_{F}^{2} + {\gamma} \left\lVert\bm{W}_{\text{M}}^{(0)}\right\rVert_{F}^{2}
\end{align}
where $\left\lVert\cdot\right\lVert_{F}$ indicates the Frobenius norm, and ${\gamma}$ is a regularization parameter. The optimal estimation \cite{brmp2021} of $\bm{W}_{\text{M}}^{(0)}$ can be found in
\begin{align}\label{eq_ls_w_base}
	\bm{\hat W}_{\text{M}}^{(0)} = (\bm{X}_{\text{M},0}^{T}\bm{X}_{\text{M},0}+\gamma \bm{I})^{-1}\bm{X}_{\text{M},0}^{T}\bm{Y}_{0}^{\text{train}}
\end{align}
in which $\cdot^{T}$ is the matrix transpose operator. Upon obtaining $\bm{\hat W}_{\text{M}}^{(0)}$, the AL-based re-training is completed.

\subsubsection{The AL-based CIL.}
Upon completing the base training, we proceed with the CIL steps using C-RLS in a recursive and analytical manner. To illustrate this, without loss of generality, assume that we are given $\mathcal{D}_{0}^{\text{train}},\dots,\mathcal{D}_{k-1}^{\text{train}}$, and let $\bm{X}_{\text{M},0:k-1}\in\mathbb{R}^{N_{0:k-1}\times d_{B}}$ and $\bm{Y}_{0:k-1}\in\mathbb{R}^{N_{0:k-1}\times \sum_{i=1}^{k-1} d_{y_{i}}}$ be the concatenated activation and label tensors respectively from phase $0$ to $k-1$, i.e.,
\begin{align}
	\resizebox{0.88\linewidth}{!}{$
		\bm{X}_{\text{M},0:k-1} = \begin{bmatrix}	\bm{X}_{\text{M},0}  \\
			\vdots \\ 	\bm{X}_{\text{M},k-1} 
		\end{bmatrix} \ \ 	\bm{Y}_{0:k-1} =
		\begin{bmatrix}
			\bm{Y}_{0}^{\text{train}}&\bm{0}& \bm{0}&\dots&\bm{0}\\
			\bm{0}&\bm{Y}_{1}^{\text{train}}&\bm{0}&\dots& \bm{0}\\
			\vdots&\vdots&&\ddots\\
			\bm{0}& \bm{0}&  \bm{0}&\dots &\bm{Y}_{k-1}^{\text{train}}\end{bmatrix}$}
\end{align}
where
\begin{align}\label{eq_trans}
	\bm{X}_{\text{M},k-1} &= \sigma_{\text{M}}(B(f_{\text{flat}}(f_{\text{CNN}}(\bm{X}^{\text{train}}_{k-1}, \bm{W}_{\text{CNN}}))))
\end{align}
and $\bm{N}_{0:k-1}$ indicates the total number of data samples from phase $0$ to $k-1$. $\bm{Y}_{0:k-1}$'s sparse structure is because data classes among different phases are mutually exclusive. The learning problem can then be extended to
\begin{align}\label{eq_k-1_L2}	\resizebox{0.9\linewidth}{!}{$\underset{\bm{W}_{\text{M}}^{(k-1)}}{\text{argmin}} \quad \left\lVert\bm{Y}_{0:k-1} - \bm{X}_{\text{M},0:k-1}\bm{W}_{\text{M}}^{(k-1)}\right\rVert_{F}^{2} + {\gamma} \left\lVert\bm{W}_{\text{M}}^{(k-1)}\right\rVert_{F}^{2}.$}
\end{align}
According to \eqref{eq_ls_w_base}, at phase $k-1$, we have
\begin{align}\label{eq_w_k-1}
	\resizebox{0.9\linewidth}{!}{$	\bm{\hat W}_{\text{M}}^{(k-1)} = \left(\bm{X}_{\text{M},0:k-1}^{T}\bm{X}_{\text{M},0:k-1}+\gamma \bm{I}\right)^{-1}\ \bm{X}_{\text{M},0:k-1}^{T}\bm{Y}_{0:k-1}$}
\end{align}
where $\bm{\hat W}_{\text{M}}^{(k-1)}\in\mathbb{R}^{d_{\text{B}}\times \sum_{i=1}^{k}d_{y_{i}}}$ whose column size is expanded as $k$ increases. Let
\begin{align}\label{eq_R_k-1}
	\bm{R}_{\text{M},k-1} = (\bm{X}_{\text{M},0:k-1}^{T}\bm{X}_{\text{M},0:k-1}+\gamma \bm{I})^{-1}
\end{align}
be an \textit{inverted auto-correlation matrix} (iACM), which captures the correlation information of both present and past samples. Building upon the above derivations, \textit{the goal of the main stream CIL is to calculate $\bm{\hat W}_{\text{M}}^{(k)}$ using only $\bm{\hat W}_{\text{M}}^{(k-1)}$, $\bm{R}_{\text{M},k-1} $ and data $\bm{X}^{\text{train}}_{k}$ of phase $k$, without involving historical samples such as $\bm{X}_{0:k-1}$.} To this end, we formulate the CIL process as a C-RLS indicated in the following theorem.

\begin{theorem}\label{thm_acil}
	Let $\bm{\hat W}_{\text{M}}^{(k)}$  be the optimal estimation of $\bm{ W}_{\text{M}}^{(k)}$  using \eqref{eq_w_k-1} with all the training data from phase $0$ to $k$. Let $\bm{\hat W}_{\text{M}}^{(k-1)\prime} = [\bm{\hat W}_{\text{M}}^{(k-1)} \ \ \ \bm{0}]$,  and $\bm{\hat W}_{\text{M}}^{(k)}$  can be equivalent calculated via
	\begin{align}\label{eq_w_update}
		\resizebox{0.88\linewidth}{!}{$\bm{\hat W}_{\text{M}}^{(k)}
			= \bm{\hat W}_{\text{M}}^{(k-1)\prime} + \bm{ R}_{\text{M},k}\bm{X}_{\text{M},k}^{T}(\bm{ Y}_{{k}}^{\text{train}} - \bm{X}_{\text{M},k}\bm{\hat W}_{\text{M}}^{(k-1)\prime})$}
	\end{align}
	where
	\begin{align}\label{eq_R_update2}
		\resizebox{0.88\linewidth}{!}{$
			\bm{R}_{\text{M},k} = \bm{R}_{\text{M},k-1} - \bm{R}_{\text{M},k-1}\bm{X}_{\text{M},k}^{T}(\bm{I} + \bm{X}_{\text{M},k}\bm{R}_{\text{M},k-1}\bm{X}_{\text{M},k}^{T})^{-1}\bm{X}_{\text{M},k}\bm{R}_{\text{M},k-1}.$}
	\end{align}
\end{theorem}
\begin{proof}
	See Supplementary material A\footnote{The supplementary materials can be found in https://github.com/ZHUANGHP/Analytic-continual-learning}.
\end{proof}

As shown in Theorem \ref{thm_acil}, the C-RLS concatenates the weight matrix (i.e., building $\bm{\hat W}_{\text{M}}^{(k-1)\prime}$) and updates it recursively. Notably,  the recursive formulas in \eqref{eq_w_update} and \eqref{eq_R_update2} are the widely recognized RLS \cite{RLS1996book}. Hence, with C-RLS, the main stream has constructed a non-forgetting CIL mechanism. Upon freezing the backbone, the CIL is equivalent to its joint-learning counterpart as shown in Theorem \ref{thm_acil}. That is, the model trained incrementally yields identical weights to that trained employing the entire data.

\subsection{The Compensation Stream of DS-AL}
The process of C-RLS is built entirely on one linear projection. When training samples are complex, under-fitting may occur. To overcome this, we introduce the compensation stream governed by the DAC module to help overcome the under-fitting limitation of linear mapping.

The compensation stream operates in a similar manner to the main stream, but it differs in that the label matrix is generated using the residue from the main stream. Without loss of generality, assume that we have executed the main stream at phase $k$ (i.e., obtaining $\bm{\hat W}_{\text{M}}^{(k)}$ and $\bm{ R}_{\text{M},k}$) and the compensation stream at phase $k-1$ (i.e., obtaining compensation weight matrix $\bm{\hat W}_{\text{C}}^{(k-1)}$ and corresponding correlation matrix $\bm{ R}_{\text{C},k-1}$). Let ${\bm{\tilde{Y}}}_{k}$ be the residue after conducting the main stream, i.e., 
\begin{align}\label{eq_res}
	{\bm{\tilde{Y}}}_{k} = [\bm{0}_{{N_{0:k-1}\times d_{y_{k-1}}}}\ \ {\bm{Y}}_{k}^{\text{train}}] - \bm{X}_{\text{M},k}\bm{\hat W}_{\text{M}}^{(k)}
\end{align}
where the zero matrix is due to the mutual-exclusive CIL setting. Let 
\begin{align}\label{eq_embed_C}
	\bm{X}_{\text{C},k} &= \sigma_{\text{C}}(B(f_{\text{flat}}(f_{\text{CNN}}(\bm{X}^{\text{train}}_{k}, \bm{W}_{\text{CNN}}))))
\end{align}
be the \textit{compensation stream activation}, where $\sigma_{\text{C}}$ is an activation function different from $\sigma_{\text{M}}$. Here $\sigma_{\text{C}}$ is a hyperparameter to be determined (e.g., can be $Tanh$, $Mish$ or others), which will be explored during the experiments.

The ${\bm{\tilde{Y}}}_{k}$ can be treated as the null space of 
$\bm{X}_{\text{M},k}$, where the embedding cannot reach. The key idea of the DAC module is to \textit{map to the null space using an alternative embedding $\bm{X}_{\text{C},k}$}, attempting to further improve the fitting ability. To achieve this, we construct a dual recursive CIL process. Similar to the C-RLS indicated in Theorem \ref{thm_acil}, the DAC module would lead to a resembling recursive structure. 

Before proceeding to the later step, a Previous Label Cleansing (PLC) process is adopted, i.e., 
\begin{align}\label{eq_plc}
	\{{\bm{\tilde{Y}}}_{k}\}_{\text{PLC}} = [\bm{0}_{{N_{0:k-1}\times d_{y_{k-1}}}}\ \ (\bm{\tilde{Y}}_{k})_{\text{new}}]
\end{align}
where $(\bm{\tilde{Y}}_{k})_{\text{new}}$ indicates a submatrix by keeping the last $d_{y_{k}}$ columns in $\bm{\tilde{Y}}_{k}$. Note that PLC does not apply during the initial phase, i.e., $	\{{\bm{\tilde{Y}}}_{0}\}_{\text{PLC}} = \bm{\tilde{Y}}_{0}$.

To illustrate the need for PLC, let $(\bm{\hat W}_{\text{M}}^{(k)})_{\text{old}}$ and $(\bm{\hat W}_{\text{M}}^{(k)})_{\text{new}}$ represent the submatrices by keeping the first $\sum_{i=0}^{k-1}d_{y_{i}}$ and the last $d_{y_{k}}$ columns respectively, namely, the weight components corresponding to the previous phases and the current phase $k$. We then rewrite \eqref{eq_res} into
\begin{align}\nonumber
	{\bm{\tilde{Y}}}_{k} &= [\bm{0}_{{N_{0:k-1}\times d_{y_{k-1}}}}\ \ {\bm{Y}}_{k}^{\text{train}}] - \bm{X}_{k}[(\bm{\hat W}_{\text{M}}^{(k)})_{\text{old}}\ \ (\bm{\hat W}_{\text{M}}^{(k)})_{\text{new}}]\\\label{eq_plc2}
	&=[-\bm{X}_{k}(\bm{\hat W}_{\text{M}}^{(k)})_{\text{old}}\ \ \ {\bm{Y}}_{k}^{\text{train}} - \bm{X}_{k}(\bm{\hat W}_{\text{M}}^{(k)})_{\text{new}}].
\end{align}
By comparing \eqref{eq_plc} with \eqref{eq_plc2}, we should find that the label using PLC in \eqref{eq_plc} is more reasonable. That is because, during the CIL, the data classes are mutually exclusive. This setting should also be applied in the proposed DAC module, otherwise \eqref{eq_plc2} could provide non-zero false supervision for previous phases (i.e., $-\bm{X}_{k}(\bm{\hat W}_{\text{M}}^{(k)})_{\text{old}}$). This will be empirically evidenced in the experiment section.

\begin{table*}[t]	
	\centering
	\fontsize{9pt}{9pt}\selectfont{
		%
			\begin{tabular}{clcccccccccccc}
				\toprule 
				\multirow{2}{*}{}&\multirow{2}{*}{Method} &\multirow{2}{*}{EFCIL?}& \multicolumn{3}{c}{\textit{CIFAR-100}} &  & \multicolumn{3}{c}{\textit{ImageNet-100}} &  & \multicolumn{3}{c}{\textit{ImageNet-Full}} \\ \cline{4-6} \cline{8-10} \cline{12-14} 
				&&& K=5       & 25    & 50   &  & K=5           & 25     & 50     &  & K=5       & 25     & 50    \\ \hline 
				\multirow{17}{*}{$\mathcal{\bar A}$}
				&LUCIR &$\times$&63.17&57.54&-&&70.84&61.44&-&&64.45&56.56&-\\
				&PODNet&$\times$&64.83&60.72&57.98&&75.54&68.31&62.48&&66.95&59.17&-\\ 
				&AANets&$\times$&66.31&62.31&-&&76.96&71.78&-&&67.73&61.78&-\\
				&RMM &$\times$&\underline{68.36}&64.12&-&&\underline{79.50}  &\underline{75.01}&-&&\underline{69.21}&63.93&-\\ 
				&FOSTER&$\times$&-&\underline{67.95}&-&&- & 69.34&-&&-&-&-\\
				\cline{2-14} 
				&LwF  &{$\checkmark$}&49.59&45.51&-&&53.62&44.32&-&&51.50&43.14&-\\
				&ACIL&{$\checkmark$}&66.30&65.95&66.01&&74.81&74.59&74.13&&65.34&64.63&64.35\\
				&PASS&{$\checkmark$}&(63.88*)&(56.86*)&(41.11*)&&72.24*&52.02*&30.59*&&-&-&-\\	
				&IL2A&{$\checkmark$}&(65.53*)&(53.15*)&(21.49*)&&-&-&-&&-&-&-\\
				&FeTrIL&{$\checkmark$}&{(66.30)}&-&-&&72.20&-&-&&66.10&-&-\\
				&iVoro-ND&{$\checkmark$}&{(67.55)}&-&-&&-&-&-&&-&-&-\\
				
				\cline{2-14} 
				
				&	\multirow{2}{*}{\textbf{DS-AL} } &	\multirow{2}{*}{$\checkmark$}& \textbf{66.39}& \underline{\textbf{66.20}}& \underline{\textbf{66.33}}&&\multirow{2}{*}{-}&\multirow{2}{*}{-}&\multirow{2}{*}{-}&&\multirow{2}{*}{-}&\multirow{2}{*}{-}&\multirow{2}{*}{-}\\
				&&&$\pm$ 0.09& $\pm$ 0.05& $\pm$ 0.11&&&&&&&&\\
				\cline{2-14} 
				
				&	\multirow{2}{*}{\textbf{DS-AL}} &	\multirow{2}{*}{$\checkmark$}
				&\underline{\textbf{{(68.39)}}}  &\underline{\textbf{{(68.40)}}} &\underline{\textbf{{(68.26)}}} &&\textbf{75.19}&\textbf{75.03}&\underline{\textbf{74.77}}&&\textbf{67.18}&\underline{\textbf{66.81}}&\underline{\textbf{66.79}}\\
				&&&$\pm$ 0.16  &$\pm$ 0.20 &$\pm$ 0.08 &&$\pm$ 0.10&$\pm$ 0.07&$\pm$ 0.11&&$\pm$ 0.03&$\pm$ 0.03&$\pm$ 0.02\\
				\hline 
				\multirow{15}{*}{$\mathcal{A}_{K}$}
				&LUCIR &$\times$&54.30&48.35&-&&60.00&49.26&-&&56.60&46.23&-\\
				&PODNet&$\times$&54.60&51.40&-&&67.60&55.34&-&&58.90&50.51&-\\ 
				&AANets&$\times$&59.39&53.55&-&&69.40&63.69&-&&60.84&53.21&-\\ 
				&RMM&$\times$&\underline{59.00}&56.50&-&&\underline{73.80}  &\underline{68.84}&-&&\underline{62.50}&55.50&-\\ \cline{2-14}
				&LwF&{$\checkmark$}&43.36  &41.66&-&&55.32  &55.12&-&&48.70  &49.84&-\\
				&ACIL&{$\checkmark$}&57.78&57.65&57.83&&66.98&67.16&67.22&&56.11&55.43&56.09\\
				&PASS&{$\checkmark$}&{({55.75*)}}&{(44.76*)}&{({28.02*)}}&&61.76*&37.46*&18.22*&&-&-&-\\
				&IL2A&{$\checkmark$}&{(53.36*)}&{(35.27*)}&{(11.03*)}&&-&-&-&&-&-&-\\	
				
				&iVoro-ND &{$\checkmark$}&{(57.25)}&-&-&&-&-&-&&-&-&-\\
				\cline{2-14} 
				&\multirow{2}{*}{\textbf{DS-AL} }  &\multirow{2}{*}{$\checkmark$}& \textbf{58.26}
				& \underline{\textbf{58.32}}& \underline{\textbf{58.37}}&&\multirow{2}{*}{-}&\multirow{2}{*}{-}&\multirow{2}{*}{-}&&\multirow{2}{*}{-}&\multirow{2}{*}{-}&\multirow{2}{*}{-}\\
				&&&$\pm$ 0.09& $\pm$ 0.12& $\pm$ 0.15&&&&&&&&\\
				\cline{2-14} 
				&\multirow{2}{*}{\textbf{DS-AL} }  &\multirow{2}{*}{$\checkmark$}&\underline{\textbf{{(61.44)}}}  &\underline{\textbf{{(61.41)}}} &\underline{\textbf{{(61.35)}}} &&\textbf{68.00}&{\textbf{67.72}}&\underline{\textbf{67.80}}&&\textbf{58.17}
				&\underline{\textbf{58.10}}&\underline{\textbf{58.15}}\\
				&&&$\pm$ 0.08  &$\pm$ 0.26 &$\pm$ 0.26 &&$\pm$ 0.17&$\pm$ 0.13&$\pm$ 0.15&&$\pm$ 0.04&$\pm$ 0.02&$\pm$ 0.03\\
				\bottomrule
		\end{tabular}}
\caption{Comparison of average accuracy $\mathcal{\bar A}$ and last-phase accuracy $\mathcal{A}_{K}$ among EFCIL and replay-based methods. Results from replay-based methods are cited from their papers. On CIFAR-100, data in bracket is for ResNet-18. Data in Bold are the best within EFCIL methods and data underlined are the best considering both categories. Data with ``*'' are those we reproduce via official codes. ``-'' means the results are not available. Lower lines in results of DS-AL are standard deviations. }
\label{table_avg_acc_ef}
\end{table*}

Upon providing the input $\bm{X}_{\text{C},k}$ and label $\{{\bm{\tilde{Y}}}_{k}\}_{\text{PLC}}$, we can proceed to recursively update the compensation weight $\bm{W}_{\text{C}}^{(k)}$ following the same C-RLS structure indicated in Theorem \ref{thm_acil}. Let $\bm{\hat W}_{\text{C}}^{(k-1)\prime} = [\bm{\hat W}_{\text{C}}^{(k-1)} \ \ \ \bm{0}]$, and we have
\begin{align}\label{eq_w_com_update}
	\resizebox{0.88\linewidth}{!}{$\bm{\hat W}_{\text{C}}^{(k)}
		= \bm{\hat W}_{\text{C}}^{(k-1)\prime} + \bm{ R}_{{\text{C},k}}\bm{X}_{\text{C},k}^{T}(\{{\bm{\tilde{Y}}}_{k}\}_{\text{PLC}} - \bm{X}_{\text{C},k}\bm{\hat W}_{\text{C}}^{(k-1)\prime})$}
\end{align}
where
\begin{align}\label{eq_R__com_update2}
	\resizebox{0.88\linewidth}{!}{$
		\bm{R}_{\text{C},k} = \bm{R}_{\text{C},k-1} - \bm{R}_{\text{C},k-1}\bm{X}_{\text{C},k}^{T}(\bm{I} + \bm{X}_{\text{C},k}\bm{R}_{\text{C},k-1}\bm{X}_{\text{C},k}^{T})^{-1}\bm{X}_{\text{C},k}\bm{R}_{\text{C},k-1}$}
\end{align}
which concludes the compensation stream.

Finally, during inference, the prediction is produced by combining both streams as follows.
\begin{align}\label{eq_inference}
	\bm{\hat Y}_{k}^{\text{(all)}} = \bm{X}_{\text{M},k}\bm{\hat W}_{\text{M}}^{(k)} + \mathcal{C}\bm{X}_{\text{C},k}\bm{\hat W}_{\text{C}}^{(k)}
\end{align}
where $\mathcal{C}$ is a compensation ratio indicating to what extent the network relies on compensation to enhance its fitting ability. The DS-AL is summarized in algorithm framework (we place the algorithm in Supplementary material B.

\section{Experiments}
In the section, we conduct experiments on various benchmark datasets, and compare the DS-AL with SOTA EFCIL methods, including LwF \cite{LwF2018TPAMI}, PASS \cite{PASS2021CVPR}, IL2A \cite{IL2A2021NeurIPS}, ACIL \cite{ACIL2022NeurIPS}, FeTrIL \cite{FeTrIL2023WACV} and iVoro-ND \cite{iVoro2023ICLR}. We also include replay-based methods, i.e., LUCIR \cite{LUCIR2019_CVPR}, PODNet \cite{podnet2020ECCV}, AANets \cite{AANet_2021_CVPR}, RMM \cite{RMM2021NeuriPS} and FOSTER \cite{FOSTER2022ECCV}. 

\subsection{Experimental Setup}
\textbf{Basic Setup.} We follow the setting from the ACIL, including datasets, architectures, training strategies, and CIL evaluation protocols. Datasets include CIFAR-100, ImageNet-100 and ImageNet-Full. Architectures are ResNet-32 on CIFAR-100 and ResNet-18 on both ImageNet-100 and ImageNet-Full, respectively. Additionally, we train a ResNet-18 on CIFAR-100 as many EFCIL methods (e.g. \cite{PASS2021CVPR} and \cite{iVoro2023ICLR}) adopt this structure.  We adopt the same training strategies as that of the ACIL (see details in Supplementary material C.

For CIL evaluation, the network is first trained (i.e., phase 0) on the base dataset containing half of the full data classes. Subsequently, the network gradually learns the remaining classes evenly for $K$ phases. Most existing methods only report small-phase results, e.g., those of $K=5,25$. We include $K=50,100, 250, 500$ as well to validate DS-AL's large-phase performance.

\textbf{Hyperparameters.} Two unique hyperparameters (i.e., $\sigma_{\text{C}}$ and $\mathcal{C}$) have been introduced in this paper. We utilize grid search to determine their values. For other hyperparameters shared among AL-based methods, we copy them from the ACIL for simplicity. The details can be found in Supplementary material C.

\subsection{Evaluation Metric}
Two metrics are adopted for evaluation. The overall performance is evaluated by the \textit{average incremental accuracy} (or average accuracy) $\mathcal{\bar A}$ (\%): $\mathcal{\bar A} = \frac{1}{	K+1}{\sum}_{k=0}^{K}\mathcal{A}_{k}$, where $\mathcal{A}_{k}$ indicates the average test accuracy at phase $k$ by testing the network on $\mathcal{D}_{0:k}^{\text{test}}$. A higher $\mathcal{\bar A}$ score is preferred when evaluating CIL algorithms. The other evaluation metric is the \textit{last-phase accuracy} $\mathcal{A}_{K}$ (\%) measuring the network's last-phase performance upon completing all CIL tasks. $\mathcal{A}_{K}$ (\%) is an important metric as reveals the gap between the CIL and the joint training, a gap all CIL methods stride to close.

\noindent\begin{minipage}[t]{0.47\textwidth}
	\vspace{0.1pt}
	\centering
	\includegraphics[width=1.0\linewidth]{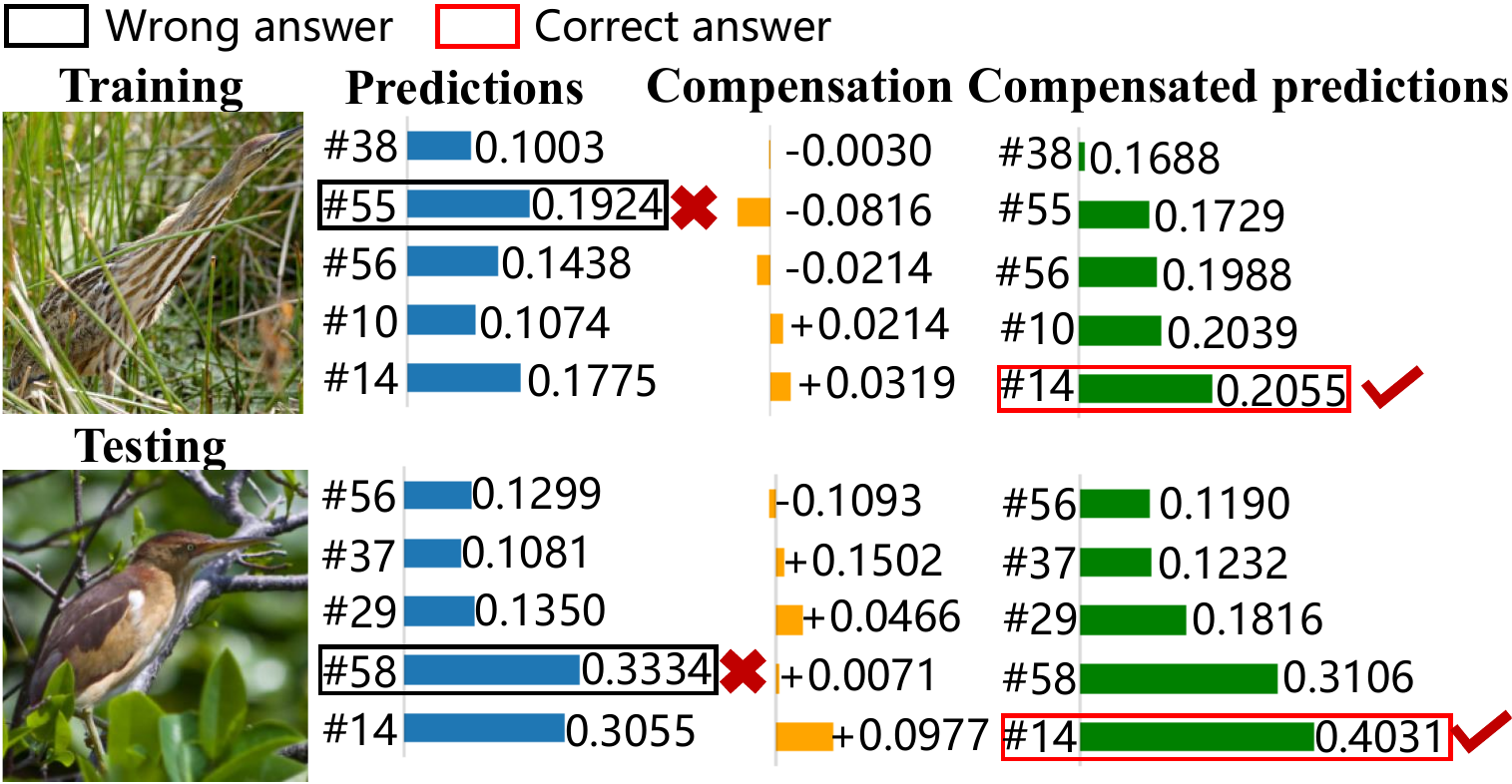}
	\captionof{figure}{The change of top-5 predictions of images w/ and w/o compensation ($\mathcal{C}=1.0$). The compensation improves the fitting as well as generalization abilities.}
	\label{fig:fitting}
\end{minipage}

\subsection{Comparison with State-of-the-arts}
We compared with both EFCIL and reply-based methods as demonstrated in Table \ref{table_avg_acc_ef}. The upper panel shows the average accuracy $\mathcal{\bar A}$ and the lower panel shows the last-phase accuracy $\mathcal{A}_{K}$. The activation function chosen in DS-AL is \textit{Tanh}. The compensation ratios $\mathcal{C}=0.6, 0.8, 1.4$  on CIFAR-100, ImageNet-100 and ImageNet-Full respectively.

\textbf{Compare with EFCIL Methods.} As shown in the upper panel in Table \ref{table_avg_acc_ef}, the DS-AL delivers the most competitive results among EFCIL methods. On CIFAR-100, it slightly outperforms the previous best technique (i.e., ACIL) by 0.48\%, 0.65\%, 0.67\% and 0.50\% of $\mathcal{A}_{K}$ for different phase settings. In particular, the performance of AL-based CIL (e.g., DS-AL and ACIL) remains roughly the same as $K$ changes, while other techniques receive degrading performance as $K$ increases. This allows our DS-AL to excel more significantly for a growing $K$. The DS-AL gives a similar performance on ImageNet-100. 

On ImageNet-Full,  the $\mathcal{\bar A}$ achieves 67.18\%, 66.91\%, 66.81\% and 66.79\%, overtaking the previous best EFCIL by \textbf{1.08}\%, \textbf{1.91}\%, \textbf{2.18}\% and \textbf{2.44}\% respectively. The DS-AL's last-phase accuracy yields 58.17\%, 58.13\%, 58.10\% and 58.15\%, leads the ACIL by \textbf{2.1}\%-\textbf{2.7}\% among various $K$ scenarios. The DS-AL has been shown to produce a more significant improvement on ImageNet-Full ($\approx2\%$ v.s. $\approx0.5\%$). This is because CIL tasks on ImageNet-Full are more challenging, requesting extra ability of fitting. Although the DS-AL and the ACIL belong to the same AL-based CIL family, the ACIL suffers from under-fitting on large-scale datasets such as ImageNet-Full.

\textbf{Compare with Replay-based Methods.}  Replay-based methods have access to historical samples during the CIL procedure, allowing an easier way to address catastrophic forgetting. As shown in Table \ref{table_avg_acc_ef}, the DS-AL runs behind replay-based counterparts for small-phase scenarios (e.g., $K=5$). For instance, the DS-AL gives 67.18\%,  worse than the 69.21\% from the RMM on ImageNet-Full. However, advantages of replaying samples are consumed as $K$ increases. For instance, for $K=25$, the DS-AL leads the best replaying method by \textbf{2.88}\% (66.81\% v.s. 63.93\%) averagely, and by \textbf{2.60}\% (58.10\% v.s. 55.50\%) at the last phase. This pattern on ImageNet-Full is consistent with those on CIFAR-100 and ImageNet-100. In general, the DS-AL begins to outperform replay-based methods from $K\ge25$, except on ImageNet-100 where the DS-AL falls behind the RMM at $K=25$, but the $\mathcal{\bar A}$ gap is very small (74.93\% v.s. 75.01\%). 

\noindent\begin{minipage}[h]{0.47\textwidth}
	\vspace{4pt}
	\centering
	\includegraphics[width=1.0\linewidth]{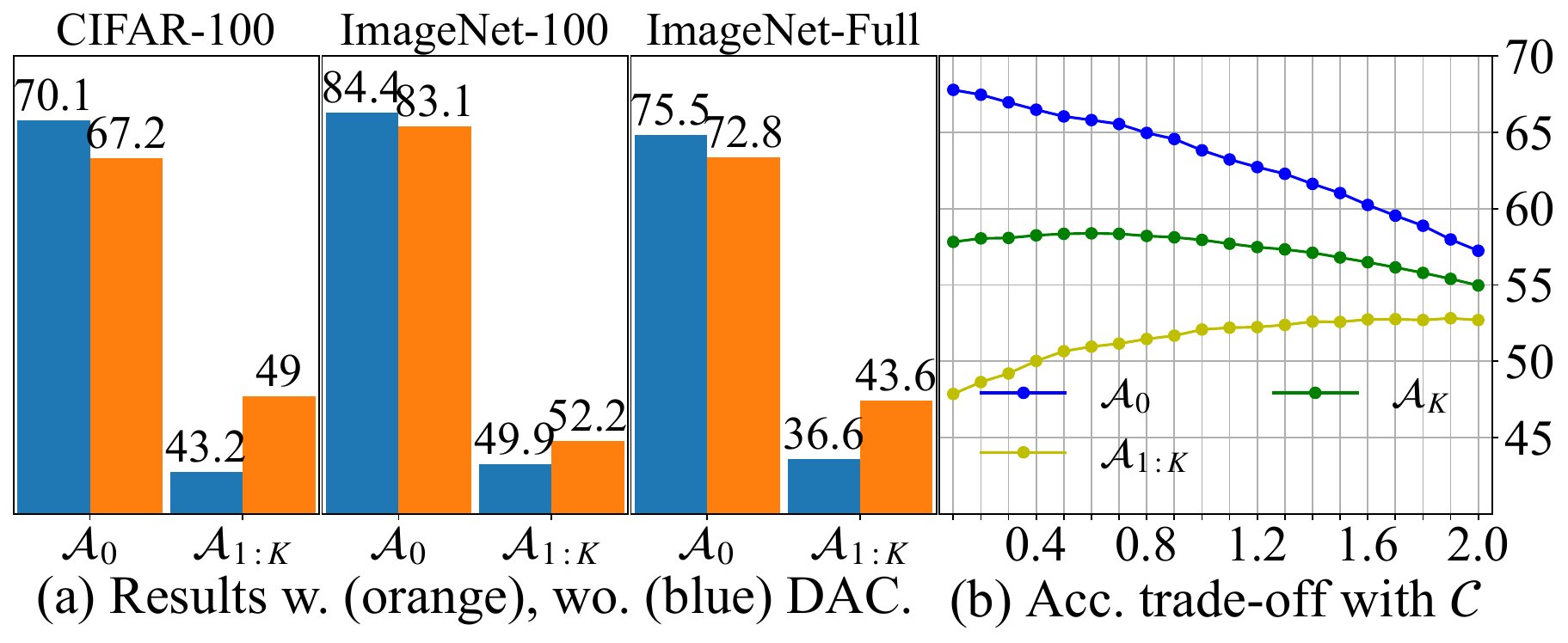}
	\captionof{figure}{(a) Stability-plasticity changes via the DAC module. (b) Analysis of compensation ratio $\mathcal{C}$ on CIFAR-100.}
	\label{fig:sta-pla}
\end{minipage}

\subsection{Analysis on Compensation Stream}
The compensation stream utilizes the DAC module to improve the main stream's fitting ability. To help understand this, we give specific examples (i.e., 5-phase experiments on ImageNet-100) during the CIL training. We plot the top-5 predictions with before and after the DAC module's contribution. As shown in the upper part of Figure \ref{fig:fitting}, the prediction for the sample class is inaccurate during training in the main stream. After applying the DAC module, the compensation stream provides an extra gain, thereby correcting the predicted results. In this example, we can observe a significant prediction change, suggesting a non-trivial enhancement on DS-AL's fitting ability, mitigating its limitation as an AL-based CIL technique. Such a fitting improvement also benefits the generalization power. As indicated in the lower panel of Figure \ref{fig:fitting}, during inference, the compensation stream can in fact correct the wrong prediction into an accurate one, especially when the top-2 predictions are comparable.

\textbf{Compensation Stream Enhances the Plasticity.} Balancing the stability (old knowledge) and plasticity (new knowledge) is crucial in CIL. Conventional AL-based CIL such as ACIL prioritizes stability as the backbone is trained completely on the base dataset. Through the compensation stream, our DS-AL can achieve a more reasonable stability-plasticity balance. To show this, we report the $\mathcal{A}_{K}$ on the base dataset (i.e., the first half) and the CIL dataset (i.e., the other half) respectively as shown in Figure \ref{fig:sta-pla}(a). By introducing the DAC module, novel classes learned incrementally receive a significant improvement (plasticity) that overtakes the performance decline on the base dataset (stability). For instance, the DAC module improves the novel class accuracy by \textbf{6.85}\% (36.61\%$\to$43.56\%) while only costing the base class accuracy to be reduced by 2.39\% (75.48\%$\to$72.79\%). This allows an improved overall accuracy, demonstrating that the DAC module is a beneficial adjustment by tuning the stability-plasticity balance. 

	\noindent\begin{minipage}[t]{0.47\textwidth}
	\vspace{0.1pt}
	\centering
	\includegraphics[width=1.0\linewidth]{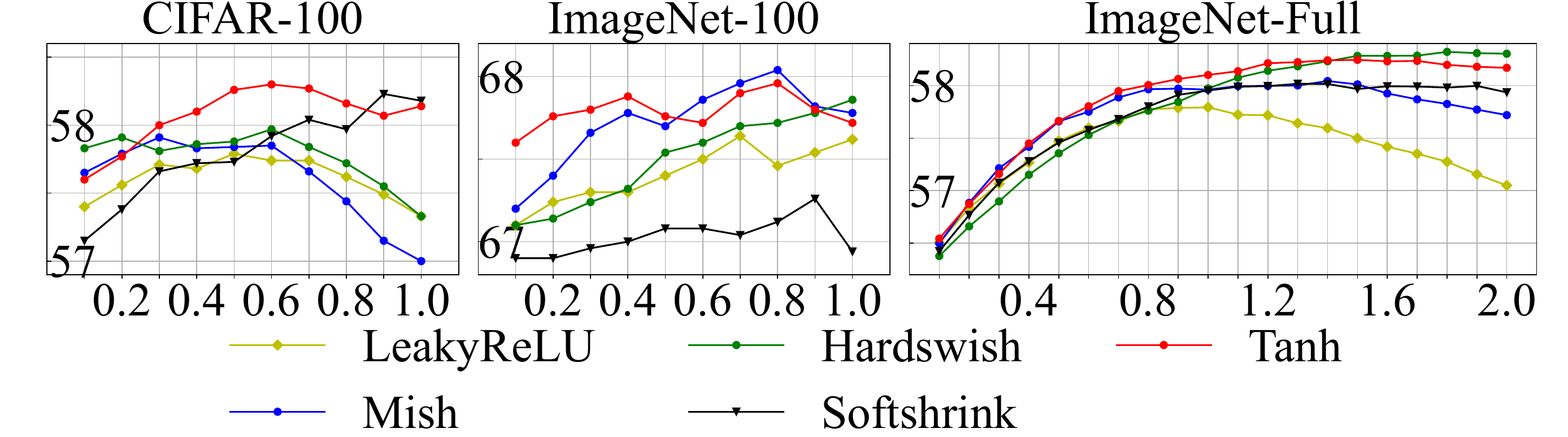}
	\captionof{figure}{Last-phase accuracy of DS-AL with different activation functions and compensation ratio.}
	\label{fig:activationfunction}
\end{minipage}

\subsection{Hyperparameter Analysis}
Activation function types and compensation ratio $\mathcal{C}$ influence the DS-AL's performance. Here we analyze these parameters in detail (with 5-phase setting). We plot the $\mathcal{A}_{K}$ with various activation functions under the scenarios of $\mathcal{C} = \left \{0.1,0.2,\dots1.0\right \}  $ in Figure \ref{fig:activationfunction} and specifically extend the $\mathcal{C}$ up to 2.0 on ImageNet-Full.

\textbf{Activation Function.} As shown in Figure \ref{fig:activationfunction}, the performance with all activation functions changes with $\mathcal{C}$. Among all candidates, we discover that the \textit{Tanh} function delivers constantly good performance across all three datasets while other functions fluctuate. Although it fails behind \textit{Mish} on ImageNet-100 and \textit{Hardwish} on ImageNet-Full, the gap is relatively negligible. This observation can be explained by the fact that the \textit{Tanh} function produces embedding values with a distribution vastly different from the main stream one (i.e., produced by \textit{ReLU}), which helps improve the fitting task to the null space of the main stream. Most commonly-seen activation functions resemble the \textit{ReLU} function curve and therefore lead to quite similar data distribution.

	%
	\textbf{Compensation Ratio.} As shown in Figure \ref{fig:activationfunction}, all results of different activation functions evolve when $\mathcal{C}$ changes. On CIFAR-100, results for most activation functions peak at the middle (e.g., $0.2<\mathcal{C}<0.8$). For example, the \textit{Tanh} reaches the top at around $\mathcal{C}=0.6$ and falls after. This is because the model needs appropriate compensation to enhance the fitting, while overcompensation may mislead the model. On ImageNet-100 and ImageNet-Full, similar patterns can be found but peak $\mathcal{C}$ values become higher in general (e.g. from $\mathcal{C}=0.6$ on CIFAR-100 to $\mathcal{C}=0.8$ and $\mathcal{C}=1.4$ of \textit{Tanh} on ImageNet-100 and ImageNet-Full). That may be caused by more challenging tasks, rendering the under-fitting issue more obvious, leading to an increase of compensation.  
	
	In addition, we explore $\mathcal{C}$'s effect on networks' stability and plasticity (see Figure \ref{fig:sta-pla}(b)). The $\mathcal{A}_0$ is found to decrease while $\mathcal{A}_{1:K}$ acts oppositely as $\mathcal{C}$ evolves, showing that more compensation will enhance the plasticity and suppress the stability. There is a leverage point of $\mathcal{C}$ after which the suppression contributes more and lead to degradation on $\mathcal{A}_{K}$.  
	
		\noindent\begin{minipage}[h]{0.47\textwidth}
		\centering
		\includegraphics[width=1\linewidth]{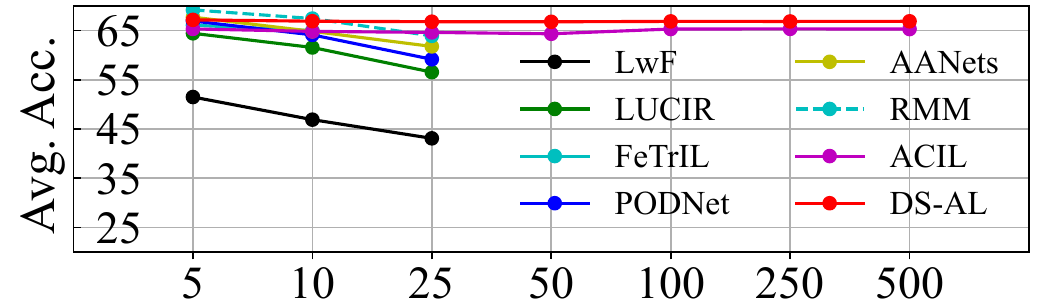}
		\captionof{figure}{The evolution of $\mathcal{\bar A}$ with the growing $K$.}
		\label{fig:K_evolution}
		\hfill
		\noindent
	\end{minipage}
	
	\subsection{Large-phase Performance} We have theoretically demonstrated that the DS-AL achieves a phase-invariant property. To empirically support our claim, we include large-phase examples with up to $K=100,250,500$ on ImageNet-Full (see Figure \ref{fig:K_evolution}). We observe that the DS-AL receives an unchanged $\mathcal{\bar A}$ across various large-phase scenarios even under the extreme case of $K=500$. As the ACIL and DS-AL both belong to the AL-based CIL, they share the same invariant property, with our method delivering a better performance. Other methods (including replay-based ones), however, lead to declining patterns as $K$ increases from $5$ to $25$. If $K$ continues to increase, a further performance reduction is expected. 
	
	\subsection{Ablation Study}
	Here we conduct an ablation study to justify the contributions of DAC and PLC. Experiments are done on ImageNet-Full with activation function \textit{Tanh} under the 5-phase setting. As shown in Table \ref{table_ablation}, performance with only the C-RLS is already competitive. However, with the DAC, the performance can be further improved. This improvement comes from the fitting and generalization enhancement from the compensation. With extra PLC that reduce the unnecessary compensation from previous classes, DS-AL can have even better performance. \\
		
				
		\noindent\begin{minipage}[h]{0.48\textwidth}	
			\renewcommand\arraystretch{1.3}
			\resizebox{0.95\textwidth}{!}{
				\begin{tabular}{lccc}
					\toprule
					Modules &$\mathcal{A_K}$ (\%)  & $\mathcal{\bar A}$ (\%) \\ 
					\hline 
					Concatenated-Recursive Least Squares (C-RLS) & 56.11 &65.34  \\ 		\hline
					+ Dual-Activation Compensation (DAC)& 57.43 &66.71  \\ 		\hline
					+ Previous Label Cleansing (PLC) & \textbf{58.17} &\textbf{67.18}
					\\
					\bottomrule
			\end{tabular} }
			\centering	
			\captionof{table}{Ablation study of the DAC  ($\mathcal{C}=1.4$) and PLC.}
			\label{table_ablation}
		\end{minipage}
		\section{Conclusion}
		In this paper, we propose a Dual-Stream Analytic Learning (DS-AL) to tackle the challenging EFCIL problem. The DS-AL comprises a main stream that formulates the CIL problem as a C-RLS solution, and a compensation stream with a DAC module to mitigate the under-fitting nature of linearity in the main stream via projecting an alternatively activated embedding onto the main stream's null space. The DS-AL establishes equivalence between CIL and its joint-learning counterpart while improving fitting power as an AL-based method. Experimental results demonstrate comparable performance to CIL across different phase counts. Introducing the compensation stream consistently enhances both fitting and generalization abilities.
		
		\section{Acknowledgements}
		This research was supported by the National Natural Science Foundation of China (Grant No. 6230070401), 2023 South China University of Technology-TCL Technology Innovation Fund, and Guangzhou Basic and Applied Basic Research Foundation (2023A04J1687). 
		\bibliography{dcacil.bib}
		
	\end{document}